\documentclass[conference]{IEEEtran}
\IEEEoverridecommandlockouts

\usepackage{cite}
\usepackage{amsmath,amssymb,amsfonts}
\usepackage{algorithmic}
\usepackage{graphicx}
\usepackage{textcomp}
\usepackage{xcolor}
\usepackage{hyperref}
\def\BibTeX{{\rm B\kern-.05em{\sc i\kern-.025em b}\kern-.08em
    T\kern-.1667em\lower.7ex\hbox{E}\kern-.125emX}}

\def \R {\mathbb{R}}

\def \T {\top}
\def \x {\times}

\def \F {\mathrm{F}}

\usepackage{amsthm}
\newtheorem{theorem}{Theorem}[section]
\newtheorem{lemma}[theorem]{Lemma}
\newtheorem{proposition}[theorem]{Proposition}
\newtheorem{corollary}[theorem]{Corollary}

\theoremstyle{definition}
\newtheorem{definition}{Definition}[section]

\begin{document}

\title{The role of class encoding in neural collapse}

\author{
\IEEEauthorblockN{Bastien Massion}
\IEEEauthorblockA{\textit{ICTEAM, UCLouvain} \\
Louvain-la-Neuve, Belgium \\
bastien.massion@uclouvain.be}
\and
\IEEEauthorblockN{Roy Makhlouf}
\IEEEauthorblockA{\textit{ICTEAM, UCLouvain} \\
Louvain-la-Neuve, Belgium \\
roy.makhlouf@uclouvain.be}
\and
\IEEEauthorblockN{Estelle Massart}
\IEEEauthorblockA{\textit{ICTEAM, UCLouvain} \\
Louvain-la-Neuve, Belgium \\
estelle.massart@uclouvain.be}
}
\maketitle

\begin{abstract}

Neural collapse is a structural property of the last-hidden-layer activations in neural network classification models, when trained beyond a zero classification error. In this work, we explore the role of label encoding in neural collapse by relying on the unrestricted feature model with mean squared error training loss. We demonstrate that, for one-hot encoded labels and balanced data, the uncentered mean features associated with each class transition from a simplex equiangular tight frame to an orthogonal frame when increasing the bias regularization coefficient associated with the final classifier. These structures are reminiscent of the orthogonal frame structure of one-hot encoded labels. For any arbitrary encoding, we also show that the final classifier's bias aims at centering the labels, compensating for the discrepancy between the global mean of the labels and the origin. We further discuss the role of the encoding in other neural collapse properties.

\end{abstract}

\begin{IEEEkeywords}
Deep learning, Neural Collapse, Class encoding
\end{IEEEkeywords}

\section{Introduction}

Neural collapse (NC) refers to a phenomenon occurring during the terminal stage of learning deep neural network classifiers. By terminal stage of learning, we refer to the last epochs, where the training loss is further reduced while the training error is zero (i.e., all inputs are correctly classified). It was noticed in \cite{Papyan2020a} that some structure appears in the activations of the last hidden layer; namely, the activations corresponding to inputs of a same class concentrate around a mean value, leading to the \emph{neural collapse} terminology, and these mean vectors, after centering, localize at the vertices of a simplex equiangular tight frame (ETF). This nice geometric structure was exploited in the design of out-of-distribution detection algorithms \cite{Haas2022} and transfer learning methods \cite{Galanti2022} among others; see \cite{Kothapalli2023} for a review on NC.

Understanding the emergence of NC has been the focus of many works over the past years. 
The original work \cite{Papyan2020a} numerically observed NC on classifiers trained with the cross-entropy loss across various datasets and architectures. Other works further explored the occurrence of NC for different losses \cite{Han2022,Zhou2022a}, and aimed at generalizing NC to intermediate layers \cite{Rangamani2023,Sukenik2023,Sukenik2024}, to regression models \cite{Andriopoulos2024} or to account for imbalanced data \cite{Fang2021a,Thrampoulidis2022,Liu2023}. Additionally, motivated by large language models, the works \cite{Jiang2024,Wu2024} investigated NC in the regime where the number of classes outnumbers the last-hidden-layer dimension.

Arguably the most common approach towards a theoretical characterization of NC relies on the \emph{unconstrained features model} (UFM) \cite{Mixon2022}, or the related \emph{layer-peeled model} \cite{Fang2021a}. These models formulate the training problem as the joint optimization of last-hidden-layer features and a final linear classifier, i.e., the last-hidden-layer features are not restricted by the expressivity of the previous layers. These simplified models, which can be seen as (overparametrized) matrix factorization problems, were used to prove that global minimizers satisfy the NC properties for various loss functions \cite{Graf2021,E2021,Han2022,Fang2021a,Mixon2022,Zhou2022a,Tirer2022,Lu2022}. Subsequent works provide global landscape analyses, showing that the UFM/layer-peeled model has a benign loss landscape and, consequently, that the training algorithm will converge to a point satisfying the NC properties as soon as it is able to avoid saddle points with a strictly negative curvature in some directions \cite{Zhu2021,Yaras2022,Zhou2022}. 

In this paper, we focus on the mean squared error (MSE) loss function. The use of this loss function gained attraction over recent years for classification problems \cite{Demirkaya2020,Hui2021}. In \cite{Tirer2022}, the authors showed that the geometric structure of the optimal solutions of the UFM for the MSE loss with one-hot encoding differs depending on the use of a bias term in the final classifier: the (uncentered) mean activations are organized as a simplex ETF when some (unregularized) bias is used, while they form an orthogonal frame (OF) in the absence of bias. In this work, we build a bridge between these two structures, by showing that OFs and simplex ETFs are merely shifted versions of one another, see Figure \ref{fig:simplex_ETF_OF_shift}. The transition between these two structures is due to the final classifier's bias, which aims at compensating for the possible non-centering of the labels and whose magnitude depends on the strength of its regularization. Note that these structures are reminiscent of the OF structure of one-hot encoded labels. We explore further the role of the encoding on NC, and demonstrate that even if variability collapse does not depend on the encoding, this is not the case of other NC properties.

The structure of the paper is as follows.  We recall fundamental definitions in Section \ref{section:preliminaries}, show the equivalence between simplex ETFs and OFs, and discuss the role of the bias of the final classifier for general encodings in Section \ref{section:bias}, and address other NC properties in Section \ref{section:NC_properties}.  

\section{Preliminaries}
\label{section:preliminaries}

Let us consider a set of input data $x_1, \dots, x_N \in \R^{d_x}$, to be classified into $K$ classes. Let us write $\bar{y}_k \in \R^K$, the label of class $k \in \{1, \dots, K\}$ and define $\bar Y:=[\bar y_1  \cdots \bar y_K ]\in\R^{K\times K}$. Note that we consider the dimension of the encoding to be equal to the number of classes $K$, which includes one-hot encoding and label smoothing, among others. Let $y_i \in \{ \bar y_1, \dots, \bar y_K\}$ be the target label vector associated with input $x_i$ and define $Y:=[y_1\cdots y_N ]\in\R^{K\times N}$. For notational convenience, we assume that the samples are ordered by class: 
\begin{equation}\label{eq:encodings_assumption}
    Y = \bar{Y}C,
\end{equation} where $C :=\begin{bmatrix}
    e_11_{n_1}^\T & \cdots & e_K1_{n_K}^\T
\end{bmatrix}\in \R^{K \x N}$ is a class assignment matrix, where $e_k$ denotes the $k^\textnormal{th}$ canonical vector, $n_k$ the number of occurrences of the target label $\bar y_k$ in the dataset, and $1_{n_k}$ the vector of ones of length $n_k$. This notation generalizes the Kronecker product formulation widely used in the literature for balanced classes \cite{Tirer2022, Tirer2023}. 

NC refers to four properties that occur simultaneously. 
\begin{itemize}
    \item $\mathcal{NC}1$ (Variability collapse): the last-hidden-layer features concentrate around their class mean.
    \item $\mathcal{NC}2$ (Convergence to a simplex ETF): the class means, after centering by their global mean, converge to a simplex ETF (see Definition \ref{def:simplex_ETF}).
    \item $\mathcal{NC}3$ (Convergence to self-duality): the last layer's classifier weights converge to the class means, up to rescaling.
    \item $\mathcal{NC}4$ (Simplification to nearest-class-center classification): the linear classifier's decision rule reduces to choosing the class whose mean is the nearest to the last-hidden-layer activations.
\end{itemize}
NC is tightly related to the notion of simplex ETF \cite{Papyan2020a}. 

\begin{definition}[Simplex ETF]\label{def:simplex_ETF}
A simplex equiangular tight frame \cite{Papyan2020a} is a collection of $K$ vectors in $\R^d$ (with $d \geqslant K$) specified by the columns of 
\begin{equation*} \label{eq:gen_simplex_ETF}
    M = \alpha P \left( I_K - \frac{1_K 1_K^\top}{K} \right),
 \end{equation*}
where $P \in \R^{d \times K}$ is semi-orthogonal (i.e., has orthonormal columns) and $\alpha >0$. Note that, equivalently, $M$ satisfies 
\begin{equation*}
    M^\top M = \alpha^2\left(I_K - \frac{1_K 1_K^\top}{K}   \right).
\end{equation*}
\end{definition}

\begin{definition}[OF] \label{def:OF}
    By orthogonal frame, we refer to a collection of $K$ vectors in $\R^d$ (with $d\geqslant K$) specified by the columns of 
    \begin{equation*}
        M=\alpha P,
    \end{equation*} where $P\in\R^{d\times K}$ is semi-orthogonal and $\alpha>0$.
\end{definition}

In this work, we analyze NC through the lens of the Unconstrained Features Model (UFM) proposed in \cite{Mixon2022}, expressed here for the MSE loss:
\begin{align}\label{eq:UFM_gen_encoding}
     \min_{W,H,b}\frac{1}{2N}&\|WH+b1_N^\top-Y\|_{\F}^2+\frac{\lambda_W}{2}\|W\|_{\F}^2 \nonumber\\&+\frac{\lambda_H}{2}\|H\|_{\F}^2+\frac{\lambda_b}{2}\|b\|_2^2,
\end{align}
where $H\in \R^{d\x N}$ with $d\geqslant K$ are unconstrained features, which model here the last-hidden-layer activations, omitting expressivity restrictions due to the network architecture, $W\in \R^{K\x d}$ and $b\in \R^K$ are respectively the final classifier's weights and bias, and $\lambda_W>0$, $\lambda_H>0$ and $\lambda_b\geqslant 0$ are regularization parameters. Note that the linearity assumption on the classifier is standard \cite{Mixon2022}.

\section{From OFs to simplex ETFs: the role of bias}
\label{section:bias}

\subsection{One-hot encoded labels and balanced data}

We address here the UFM problem \eqref{eq:UFM_gen_encoding} with one-hot encoded labels and balanced data, in line with \cite{Tirer2022,Zhou2022}.

\begin{theorem}[Theorem 3.1 in \cite{Zhou2022}, shortened]\label{thm:global_mins_one_hot_encoded}
    Assume that the labels are one-hot encoded i.e., $\Bar{Y}=I_K$, that the data are balanced, i.e., $n_1=\cdots=n_K=:n=\frac{N}{K}$, and define $c:=K\sqrt{n\lambda_W\lambda_H}$. Then, if $c<1$, any global minimizers $(W^*, H^*, b^*)$ of \eqref{eq:UFM_gen_encoding} satisfy $\mathcal{NC}1$ and $\mathcal{NC}3$: 
    \begin{equation} \label{eq:NC1_NC3_one_hot_encoded}
        H^*=\bar H^* C\quad\text{and}\quad {W^*}^\T=\sqrt{\frac{n\lambda_H}{\lambda_W}}\bar H^{*},
    \end{equation}
    where $\bar H^*\in\R^{d\times K}$ is such that 
    \begin{equation}\label{eq:NC2_one_hot_encoded}
        \bar H^{*\top}\bar H^*=\begin{cases}
        \alpha_1\left(I_K-\frac{1_K1_K^\top}{K}\right) & \text{if $\lambda_b\leqslant \frac{c}{1-c}$,}\\
        \alpha_2\left(I_K-\frac{c}{\lambda_b (1-c)}\frac{1_K1_K^\top}{K}\right)&\text{otherwise,}
    \end{cases}
    \end{equation} 
    for some $\alpha_1>0$ and $\alpha_2>0$ that depend on $\lambda_W$, $\lambda_H$ and $\lambda_b$. Moreover, the optimal bias $b^*$ satisfies \begin{equation}\label{eq:bias_one_hot_encoded}
            b^*=\begin{cases}
            \frac{1}{1+\lambda_b}\frac{1_K}{K} & \text{if $\lambda_b\leqslant \frac{c}{1-c}$,}\\
            \frac{c}{\lambda_b}\frac{1_K}{K} &\text{otherwise.}
        \end{cases}
    \end{equation}
\end{theorem}

Theorem \ref{thm:global_mins_one_hot_encoded} reveals a transition in the structure of $\bar H^*$ as $\lambda_b$ decreases, shifting from an OF (for $\lambda_b\to\infty$) to a simplex ETF (for $\lambda_b \leq c/(1-c)$). This observation suggests a link between OFs and simplex ETFs. Indeed,  centering any orthogonal frame $M = \alpha P$ gives $$M_c=\alpha P - (\alpha P)\frac{1_K1_K^\top}{K},$$ which, according to Definition \ref{def:simplex_ETF}, is a simplex ETF; see Figure \ref{fig:simplex_ETF_OF_shift} for an illustration. This led us to propose the following definition.
\begin{figure}
    \centering
    \includegraphics[width=0.8\linewidth]{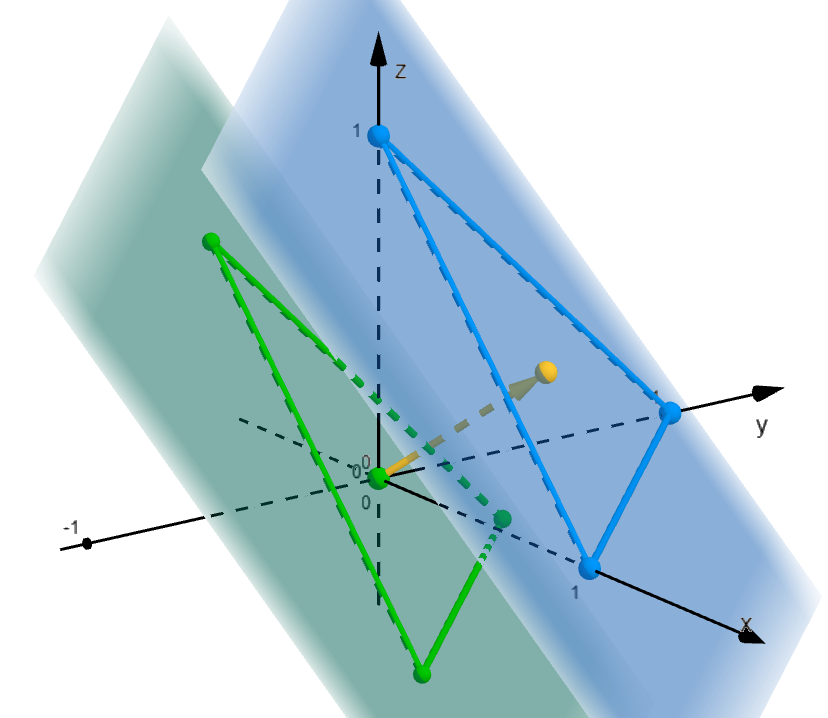}
    \caption{Centering an OF (blue) always results in a simplex ETF (green). For this figure, the parameters in Definitions \ref{def:simplex_ETF} and \ref{def:OF} are chosen as $d=K=3$, $P=I_K$ and $\alpha = 1$.}
    \label{fig:simplex_ETF_OF_shift}
\end{figure}

\begin{definition}(Shifted simplex ETF) \label{def:shifted_simplex_ETF}
A shifted simplex equiangular tight frame (SSETF) is a collection of $K$ vectors in $\R^d$ (with $d \geqslant K$) specified by the columns of \begin{equation*} \label{eq:shifted_simplex_ETF}
    M = \alpha P \left( I_K - \beta\frac{1_K 1_K^\top}{K} \right),
\end{equation*} where $P\in \R^{d\x K}$ is semi-orthogonal, $\alpha >0$ and $\beta \in \R$.
\end{definition}

This proposed structure generalizes both OFs ($\beta = 0$) and simplex ETFs ($\beta = 1$). Note that centering any SSETF leads to a simplex ETF. 

\begin{proposition}\label{prop:centering_SSETF}
If $M \in \R^{d\x K}$ is an SSETF, then $M_c = M\left(I_K-\frac{1_K1_K^\T}{K}\right)$ is a simplex ETF.
\end{proposition}
\begin{proof}
By Definition \ref{def:shifted_simplex_ETF}, $M = \alpha P ( I_K - \beta\frac{1_K 1_K^\top}{K} )$ for some $\alpha, \beta \in \R$. By definition, $M_c = M(I_K-\frac{1_K1_K^\T}{K})$, which simplifies to  $ M_c =  \alpha P(  I_K - \frac{1_K 1_K^\top}{K})$ for any $\beta$, i.e., $M_c$ is a simplex ETF. 
\end{proof}

Notice in \eqref{eq:bias_one_hot_encoded} how the optimal bias $b^*$ evolves as $\lambda_b$ decreases: it transitions continuously from $0_K$ in the bias-free UFM ($\lambda_b\to\infty)$ to $\frac{1}{K}1_K$ in the unregularized-bias UFM ($\lambda_b=0$), namely, the mean of the labels. In other words, for one-hot encoded labels, the bias in the unregularized-bias UFM compensates for the non-centering of the labels. 

\subsection{Arbitrary labels}

We generalize here the results of previous section to arbitrary label encoding $\bar Y\in\R^{K\times K}$ and imbalanced data. In this setting, the authors of \cite{Zhou2022} showed that, for the bias-free UFM, the optimal unconstrained feature matrix $H^*$ is closely related to the right singular vectors of $Y$, i.e., depends on the encoding.  
To address the UFM with bias \eqref{eq:UFM_gen_encoding}, we first derive optimality conditions in Lemma \ref{lem:optimality_conditions}, from which  Theorem \ref{thm:importance_of_bias} is proven.
\begin{lemma}\label{lem:optimality_conditions}
Let $\bar Y \in \R^{K \times K}$ be an arbitrary encoding matrix and let $Y=\bar Y C \in \R^{K \x N}$ for some class assignment matrix $C$, see \eqref{eq:encodings_assumption}. Then, any global minimizer $(W^*,H^*, b^*)$ of \eqref{eq:UFM_gen_encoding} satisfies the first-order optimality conditions:
\begin{align}
    YM_{\gamma}{H^*}^\T&=W^*H^*M_{\gamma}{H^*}^\T+\lambda_W N W^*, \label{eq:stat_cond_1} \\
    {W^*}^\T Y M_{\gamma}&={W^*}^\T W^* H^*M_{\gamma} + \lambda_H NH^*, \label{eq:stat_cond_2} \\
    b^* &= \frac{\gamma}{N}(Y-W^*H^*)1_N, \label{eq:stat_cond_3}
\end{align}
where $\gamma = \frac{1}{1+\lambda_b}$ and $M_{\gamma} = I_N -\gamma\frac{1_N1_N^\T}{N}$.
\end{lemma}
\begin{proof}
Since \eqref{eq:UFM_gen_encoding} is strictly convex and coercive in $b$, setting the gradient w.r.t. $b$ of \eqref{eq:UFM_gen_encoding} to $0$ allows us to find $b^*$, leading to \eqref{eq:stat_cond_3}. Then, requiring the gradients w.r.t. $W$ and $H$ to be $0$ give the stationarity conditions \eqref{eq:stat_cond_1} and \eqref{eq:stat_cond_2}, after plugging in \eqref{eq:stat_cond_3}. As a result, any global minimizer $(W^*, H^*, b^*)$ of \eqref{eq:UFM_gen_encoding} must satisfy \eqref{eq:stat_cond_1}-\eqref{eq:stat_cond_3}.
\end{proof}

\begin{theorem}\label{thm:importance_of_bias}
    Let $\bar Y \in \R^{K \times K}$ be an arbitrary encoding matrix and let $Y=\bar Y C \in \R^{K \x N}$ for some class assignment matrix $C$, see \eqref{eq:encodings_assumption}. 
    \begin{itemize}
        \item The UFM \eqref{eq:UFM_gen_encoding} with $\lambda_b=0$ has optimal bias $b^*=Y\frac{1_N}{N}$. It reduces to the UFM with no bias for centered labels $Y-Y\frac{1_N1_N^\top}{N}$.
        \item If $Y$ is centered, i.e., $Y1_N=0_K$, then for all $\lambda_b\geqslant 0$, the optimal bias of the UFM \eqref{eq:UFM_gen_encoding} is $b^*=0_K$.
    \end{itemize}
\end{theorem}

\begin{proof}
    In Lemma \ref{lem:optimality_conditions}, $M_{\gamma}$ satisfies $M_{\gamma}1_N = (1-\gamma)1_N$. Multiplying \eqref{eq:stat_cond_2} on the right by $1_N$ thus gives
    \begin{equation*}
        (1-\gamma){W^*}^\top Y1_N=((1-\gamma){W^*}^\top W^*+\lambda_H N I_d)(H^*1_N).
    \end{equation*} 
    If $\lambda_b=0$ (i.e., $\gamma=1$) or $Y$ is centered, the left-hand side of the above equation vanishes. Moreover, since $(1-\gamma){W^*}^\top W^*+\lambda_H N I_d$ is always invertible for $\lambda_H>0$ and $\lambda_b\geqslant 0$, it follows that $H^*1_N=0_K$ in both cases. Hence, by \eqref{eq:stat_cond_3}, $b^*=\frac{\gamma}{N}Y1_N$. On the one hand, if $\lambda_b=0$, the bias simplifies to $b^*=\frac{1}{N}Y1_N$ and the objective in \eqref{eq:UFM_gen_encoding} becomes
    \begin{align*}
        \frac{1}{2N}\biggl\|WH+Y\frac{1_N1_N^\T}{N}-Y\biggr\|_{\F}^2 + \frac{\lambda_W}{2}\|W\|_{\F}^2 + \frac{\lambda_H}{2}\|H\|_{\F}^2.
    \end{align*}
    Thus, solving the UFM with $\lambda_b=0$ and labels $Y$ reduces to solving the bias-free UFM with centered labels $Y-Y\frac{1_N1_N^\T}{N}$. On the other hand, if $Y$ is centered, i.e., $Y1_N= 0_K$, then $b^* = 0_K$, regardless of the value of $\lambda_b\geqslant 0$.
\end{proof}

Theorem \ref{thm:importance_of_bias} confirms that the bias in the UFM aims at centering the labels beyond one-hot encoding and balanced data.

\section{Impact of encoding on NC properties}
\label{section:NC_properties}
We explore in this section the impact of the encoding on the NC properties, from the viewpoint of the UFM with MSE loss. First, we show that variability collapse $\mathcal{NC}1$ holds for any encoding. Then, for balanced scenarios, we relate convergence of the centered features to a simplex ETF (i.e., $\mathcal{NC}2$) to SSETF encodings and demonstrate that a weaker version of self-duality, denoted $\mathcal{NC}3'$, is preserved regardless of the encoding.

\subsection{Robustness of variability collapse  ($\mathcal{NC}1$)}

$\mathcal{NC}1$ has been shown to hold for multiple instances of the UFM under MSE loss, such as in the bias-free case with one-hot encoded labels and imbalanced data \cite{Liu2024} or for $\lambda_b\geqslant 0$, with one-hot labels and balanced data \cite{Zhou2022}. To our knowledge, our next result is the first to ensure that $\mathcal{NC}1$ holds more generally for any bias regularization $\lambda_b \geqslant 0$ and for any encoding as well as for imbalanced scenarios. The independence of variability collapse on the encoding and on the value of $\lambda_b$ can be explained intuitively: assuming that the model has found an optimal feature representation for a sample of a given class, then all samples of the same class should be represented identically, because of the separability of the problem regarding the samples.

\begin{theorem}\label{thm:variabililty_collapse}
Let $\bar Y \in \R^{K \times K}$ be an arbitrary encoding matrix and let $Y=\bar Y C \in \R^{K \x N}$ for some class assignment matrix $C$, see \eqref{eq:encodings_assumption}. Assume that classes are balanced, and let $(W^*, H^*, b^*)$ be an optimal solution of \eqref{eq:UFM_gen_encoding}. Then, for any $\lambda_H, \lambda_W >0$ and $\lambda_b\geqslant 0$, $H^*$ satisfies $\mathcal{NC}1$, i.e., \begin{equation}\label{eq:H_bar_def}
    H^* = \Bar{H}^*C
\end{equation} for some $\Bar{H}^*\in\R^{d\times K}$.
\end{theorem}

\begin{proof}
Let $h_{k,i}$ be the feature vector corresponding to the $i^\textnormal{th}$ sample of class $k$, with $1\leqslant i \leqslant n_k$ and $1\leqslant k \leqslant K$. We decompose the two terms involving $H$ in \eqref{eq:UFM_gen_encoding} as
\begin{align*}
    \frac{1}{2N}\sum_{k=1}^K\sum_{i=1}^{n_k}\|Wh_{k,i} + b -\Bar{y}_k\|_{2}^2 + \frac{\lambda_H}{2}\sum_{k=1}^K\sum_{i=1}^{n_k}\|h_{k,i}\|_{2}^2.
\end{align*}
For the $k^{th}$ sum of the second term, we get
\begin{equation} \label{eq:decomp_1}
    \sum_{i=1}^{n_k}\|h_{k,i}\|_{2}^2 \geqslant \frac{1}{n_k}\biggl\|\sum_{i=1}^{n_k}h_{k,i}\biggr\|_{2}^2 =  n_k\|\bar{h}_k\|_{2}^2,
\end{equation}
where we first used Jensen's inequality applied to the convex function $f(x)=\|x\|_2^2$ for $x\in \R^{d}$, and where we defined $\bar{h}_k := (\sum_{i=1}^{n_k}h_{k,i})/n_k$. Note that, for each class $k$, equality can always be reached in \eqref{eq:decomp_1} and holds if and only if $h_{k,i}= \bar{h}_k$ for all $1\leqslant i\leqslant n_k$. Then, we apply the same reasoning for the $k^{\textnormal{th}}$ sum of the first term:
\begin{equation}\label{eq:decomp_2}
    \sum_{i=1}^{n_k}\|Wh_{k,i} + b -\Bar{y}_k\|_{2}^2 \geqslant  n_k\|W\bar{h}_k + b -\Bar{y}_k\|_{2}^2,
\end{equation}
with equality in \eqref{eq:decomp_2} if and only if $Wh_{k,i} + b -\bar{y}_k = W\bar{h}_k + b - \bar{y_k}$ for all $1\leqslant i \leqslant n_k$, or equivalently $Wh_{k,i}  = W\bar{h}_k$. These conditions are weaker than the previously derived conditions $h_{k,i}= \bar{h}_k$, and consequently do not impose any additional constraints on $H$. Then, any optimal $H^*$ in \eqref{eq:UFM_gen_encoding} can be written as $H^* = \Bar{H}^*C$ with $\Bar{H}^* = [\bar h_1 \cdots \bar h_K]$.  
\end{proof}

\subsection{Convergence to a simplex ETF ($\mathcal{NC}2$) for SSETF encodings}

Currently, the complete determination of which encodings result in $\mathcal{NC}2$ for the UFM with MSE loss remains an open question, both in the balanced and imbalanced cases. We next show that a sufficient condition for $\mathcal{NC}2$ is the fact that the optimal mean activations $\Bar{H}^*$ form an SSETF; this directly follows from Proposition \ref{prop:centering_SSETF}. 

\begin{corollary}\label{cor:NC2}
Let $\bar Y\in \R^{K\x K}$ be an arbitrary class encoding matrix. If, for any $\lambda_W, \lambda_H>0$ and $\lambda_b \geqslant 0$, the optimal mean features $\Bar{H}^*$ of \eqref{eq:UFM_gen_encoding} defined in  \eqref{eq:H_bar_def} form an SSETF, then their centered counterpart $\Bar{H}^*_c =\Bar{H}^* (I_K-\frac{1_K 1_K^\top}{K})$ form a simplex ETF, in agreement with $\mathcal{NC}2$.
\end{corollary}

This corollary shows that the optimal solutions of the UFM with MSE loss for balanced one-hot encoded data, provided in Theorem \ref{thm:global_mins_one_hot_encoded}, satisfy $\mathcal{NC}2$, as empirically observed by the authors of \cite{Tirer2022}. Future work will aim at deriving necessary and sufficient conditions on the encoding matrix $\bar Y\in \R^{K\x K}$ to result in mean features $\Bar{H}^*$ with an SSETF structure, and consequently, to lead in $\mathcal{NC}2$. We conjecture that, in balanced scenarios, $\Bar{H}^*\in \R^{d\x K}$ is an SSETF if and only if the encoding $\bar{Y}\in \R^{K \x K}$ is an SSETF. The family of SSETFs encompasses two widely used encodings: one-hot encoding (for $\beta=0$, $\alpha=1$, $d=K$ and $P=I_K$ in Definition \ref{def:shifted_simplex_ETF}), and label smoothing, a ``smoothed'' variant of one-hot encoding. Indeed, smoothed labels can be written as $\Bar{Y}^{\delta} = (1-\delta)I_K + \frac{\delta}{K}1_K1_K^\T$ for some smoothing parameter $0<\delta <1$ \cite{Szegedy2016,Guo2025}, and form an SSETF; simply take $\alpha= 1-\delta$, $\beta = \frac{\delta}{1-\delta}$, $d=K$ and $P=I_K$ in Definition \ref{def:shifted_simplex_ETF}. While we showed earlier that $\mathcal{NC}2$ holds for one-hot encoded labels (see Corollary \ref{cor:NC2}), this is still unproven for label smoothing, though empirical results supporting \ref{cor:NC2} were obtained using the cross-entropy loss \cite{Guo2025}.

\subsection{Failure of self-duality ($\mathcal{NC}3$) and success of a weaker form of self-duality ($\mathcal{NC}3'$)}

For balanced data and one-hot encoded labels, Theorem \ref{thm:global_mins_one_hot_encoded} showed that ${W^*}^\T$ and $\bar H^*$ are necessarily aligned, i.e., the global minimizers of the UFM satisfy $\mathcal{NC}3$. We show here that $\mathcal{NC}3$ may fail, depending on the encoding. In particular, applying an orthogonal transformation $Q$ to the labels destroys the alignmnent between ${W^*}^\T$ and $\bar H^*$.

\begin{proposition}\label{prop:loss_of_self_duality}
    Let $\bar Y \in \R^{K \times K}$ be arbitrary and let $Y=\bar Y C \in \R^{K \x N}$ for some class assignment matrix $C$, see \eqref{eq:encodings_assumption}. Let $Q\in\R^{K\times K}$ be an orthogonal matrix and let $Y_Q=QY$. Then, the global minimizers of the UFM with labels $Y_Q$ are of the form $(QW^*, H^*, Qb^*)$, where $(W^*, H^*, b^*)$ is any global minimizer of the UFM with labels $Y$.
\end{proposition}

\begin{proof}
    Replacing $(W, H, b)$ in the UFM \eqref{eq:UFM_gen_encoding} with modified labels $Y_Q$ by $(QW', H, Qb')$ and using the invariance of the Frobenius norm under multiplication with an orthogonal matrix yields the result. Since $Q$ is invertible, this change of variables is reversible, providing a bijection between the set of global minimizers of the original problem and those of the problem with new labels.
\end{proof}

The optimal classification weights $W^*$ (and the bias $b^*$) directly follow the encoding $Y$ undergoing any left orthogonal transformation, indicating that the left singular vectors of $W^*$ are strongly related to the ones of $\Bar{Y}$. This observation matches the result from \cite{Zhou2022} in the bias-free case. Meanwhile, the mean class features $\Bar{H}^*$ stay unchanged, so the self-duality between ${W^*}^\T$ and $\Bar{H}^*$ is broken. However, inverting the isometry reconstructs their alignment, indicating that the property is not denied at a more fundamental level. Based on this remark, we propose the following weaker variant of $\mathcal{NC}3$, that we will show to hold in a broader setting. 
\begin{itemize}
    \item $\mathcal{NC}3'$ (Variant of $\mathcal{NC}3$): the last layer's classifier weights converge to the class means, up to rescaling and rotation/reflections.
\end{itemize}

Mathematically, $\mathcal{NC}3'$ is satisfied if and only if there exists an orthogonal matrix $Q\in \R^{K \x K}$ and a scalar $\kappa>0$ such that 
\begin{equation} \label{eq:NC3bis}
    {W^*}^\T = \kappa \Bar{H}^*Q^\T.
\end{equation}
This equality can be equivalently stated in terms of the Gram matrices of ${W^*}^{\T}$ and $\Bar{H}^*$:
\begin{equation*}
    {W^*}^\T W^* = \kappa^2 \Bar{H}^*{\Bar{H}}^{*^\T}.
\end{equation*} 

As stated in Theorem \ref{thm:Gram_self_duality}, $\mathcal{NC}3'$ is always verified in the UFM with MSE loss and balanced classes. Besides, the theorem implies that there always exists an orthogonal matrix $Q\in \R^{K\x K}$ such that it reverts to the original self-duality property $\mathcal{NC}3$ when considering the labels $Q^\T Y$. 

\begin{theorem}\label{thm:Gram_self_duality}
Let $\bar Y\in \R^{K\x K}$ be an arbitrary class encoding matrix, and assume that the data are balanced. Then, for any $\lambda_W, \lambda_H>0$ and $\lambda_b \geqslant 0$, the optimal mean features $\Bar{H}^*$ of \eqref{eq:UFM_gen_encoding} defined in \eqref{eq:H_bar_def} satisfy $\mathcal{NC}3'$, with scaling coefficient $\kappa=\sqrt{\frac{\lambda_Hn}{\lambda_W}}$ in \eqref{eq:NC3bis}. 
\end{theorem}
\begin{proof}
Multiply the optimality conditions \eqref{eq:stat_cond_1} and \eqref{eq:stat_cond_2} respectively by ${W^*}^\T$ on the right and by ${\Bar{H}}^{*^\T}$ on the left. Subtracting these two equations gives $\lambda_W N {W^*}^\T W^* = \lambda_H N H^*H^{*^\T}$. By Theorem \ref{thm:variabililty_collapse}, we can substitute $H^* = \Bar{H}^*C$, where $C$ is the class assignment matrix. Finally, leveraging the equality $CC^\T = nI_K$ for balanced classes yields ${W^*}^\T W^* = \frac{\lambda_H n}{\lambda_W} \Bar{H}^*\Bar{H}^{*^\T}$, where we identify $\kappa^2 = \frac{\lambda_Hn}{\lambda_W}$.
\end{proof}

\section{Conclusion}

This work studies the role of class encoding in NC. It demonstrates that within the UFM under MSE loss, the final classifier's bias aims at compensating for the non-centering of the labels. In the case of balanced, one-hot encoded labels, increasing the bias regularization parameter yields a continuous transition of the unconstrained mean features from an OF to an ETF. We finally discuss the dependency of other NC properties on class encoding.

\section*{Acknowledgment}
R. M. is a FRIA grantee of the Fonds de la Recherche Scientifique - FNRS. E. M. work is partly funded by the FRS-FNRS Research Project NTTN (grant number TW02223) and the Concerted Research Action (ARC) “Gravit-AI".

\bibliographystyle{plain}

\bibliography{biblio}

\end{document}